\documentclass{article}
\usepackage[accepted]{icml2020}
\usepackage[utf8]{inputenc}
\usepackage[T1]{fontenc}
\usepackage[english]{babel}
\usepackage[colorlinks]{hyperref}
\usepackage{url}
\usepackage{amsfonts}
\usepackage{nicefrac}
\usepackage{microtype}
\usepackage{algorithm, algorithmic}
\usepackage{wrapfig}
\usepackage{amsmath}
\usepackage{amssymb}
\usepackage{mathtools}
\usepackage{subcaption}
\usepackage{ragged2e}
\usepackage{booktabs}
\usepackage{tikz}
\usepackage{etoolbox}
\usepackage{xspace}
\usepackage{xpatch}
\usepackage{enumerate}
\usepackage{xstring}
\usepackage{amsthm}

\usepackage{wrapfig}

\usepackage{setspace}
\usepackage{tabularx}
\usepackage{graphicx}
\usepackage[capitalise,noabbrev,nameinlink]{cleveref}
\usepackage[useregional=numeric]{datetime2}

\setcitestyle{authoryear,round,citesep={;},aysep={,},yysep={;}}
\renewcommand{\cite}[1]{\citep{#1}}

\creflabelformat{equation}{#2\textup{#1}#3}
\crefname{algocf}{Algorithm}{Algorithms}
\Crefname{algocf}{Algorithm}{Algorithms}

\definecolor{mydarkblue}{rgb}{0,0.08,0.45}
\hypersetup{%
colorlinks=true,
linkcolor=mydarkblue,
citecolor=mydarkblue,
filecolor=mydarkblue,
urlcolor=mydarkblue}

\graphicspath{{figures/}}

\DeclarePairedDelimiterX{\SquareBrackets}[1]{[}{]}{#1}
\DeclarePairedDelimiterX{\RoundBrackets}[1]{(}{)}{#1}
\DeclarePairedDelimiterX{\DivergenceBrackets}[2]{[}{]}{#1\;\delimsize\|\;#2}

\NewDocumentCommand{\pr}{ O{p} r() }{
  \def\prArg{#2}\patchcmd{\prArg}{|}{\mid}{}{}#1\RoundBrackets{\prArg}}
\NewDocumentCommand{\p}{ r() }{\pr[p](#1)}
\NewDocumentCommand{\q}{ r() }{\pr[q](#1)}
\NewDocumentCommand{\Normal}{ r() }{\pr[\operatorname{Normal}](#1)}
\NewDocumentCommand{\Cat}{ r() }{\pr[\operatorname{Cat}](#1)}
\NewDocumentCommand{\Bin}{ r() }{\pr[\operatorname{Bin}](#1)}
\NewDocumentCommand{\Beta}{ r() }{\pr[\operatorname{Beta}](#1)}
\NewDocumentCommand{\Bernoulli}{ r() }{\pr[\operatorname{Bernoulli}](#1)}
\NewDocumentCommand{\Dir}{ r() }{\pr[\operatorname{Dir}](#1)}

\newlength\widthE

\icmltitlerunning{Goal-Aware Prediction: Learning to Model What Matters}

\usepackage{titlesec}
\titlespacing\section{0pt}{3pt plus 2pt minus 2pt}{0pt plus 2pt minus 2pt}
\titlespacing\subsection{0pt}{3pt plus 4pt minus 2pt}{0pt plus 2pt minus 2pt}
\titlespacing\subsubsection{0pt}{3pt plus 4pt minus 2pt}{0pt plus 2pt minus 2pt}
\newcommand\blfootnote[1]{%
  \begingroup
  \renewcommand\thefootnote{}\footnote{#1}%
  \addtocounter{footnote}{-1}%
  \endgroup
}

\begin{document}

\twocolumn[
\icmltitle{Goal-Aware Prediction: Learning to Model What Matters}
\begin{icmlauthorlist}
\icmlauthor{Suraj Nair}{to}
\icmlauthor{Silvio Savarese}{to}
\icmlauthor{Chelsea Finn}{to}
\end{icmlauthorlist}
\icmlaffiliation{to}{Stanford University}
\icmlcorrespondingauthor{Suraj Nair}{surajn@stanford.edu}
\vskip 0.3in
]
\printAffiliationsAndNotice

\begin{abstract}
Learned dynamics models combined with both planning and policy learning algorithms have shown promise in enabling artificial agents to learn to perform many diverse tasks with limited supervision. However, one of the fundamental challenges in using a learned forward dynamics model is the mismatch between the objective of the learned model (future state reconstruction), and that of the downstream planner or policy (completing a specified task). This issue is exacerbated by vision-based control tasks in diverse real-world environments, where the complexity of the real world dwarfs model capacity. In this paper, we propose to direct prediction towards task relevant information, enabling the model to be aware of the current task and encouraging it to only model relevant quantities of the state space, resulting in a learning objective that more closely matches the downstream task. Further, we do so in an entirely self-supervised manner, without the need for a reward function or image labels. We find that our method more effectively models the relevant parts of the scene conditioned on the goal, and as a result outperforms standard task-agnostic dynamics models and model-free reinforcement learning.
\end{abstract}
\section{Introduction}
\label{sec:intro}

Enabling artificial agents to learn from their prior experience and generalize their knowledge to new tasks and environments remains an open and challenging problem. Unlike humans, who have the remarkable ability to quickly generalize skills to new objects and task variations, current methods in multi-task reinforcement learning require heavy supervision across many tasks before they can even begin to generalize well. One way to reduce the dependence on heavy supervision is to leverage data that the agent can collect autonomously without rewards or labels, termed \textit{self-supervision}. 
One of the more promising directions in learning transferable knowledge from this unlabeled data
lies in learning the dynamics of the environment, as the physics underlying the world are often consistent across scenes and tasks. However learned dynamics models do not always translate to good downstream task performance, an issue which we study and attempt to mitigate in this work.

While learning dynamics in low-dimensional state spaces has shown promising results \cite{McAllister2016ImprovingPW, Deisenroth11pilco:a, chua_handful, awarenessmodels_amos}, scaling to high dimensional states, such as image observations,
poses numerous challenges. One key challenge is that, in high dimensional spaces, learning a perfect model is often impossible due to limited model capacity, 
and as a result downstream task-specific planners/policies struggle from inaccurate model predictions. Specifically, a learned planner/policy will often exploit errors in the model that make it drastically overestimate its performance.
Furthermore, depending on the nature of the downstream task, prediction accuracy on certain states may be more important, which is not captured by the next state reconstruction objective used to train forward dynamics models.
\blfootnote{Videos/code can be found at \url{https://sites.google.com/stanford.edu/gap}}

Our primary hypothesis is that this ``objective mismatch'' between the training objective of the learned model (future state reconstruction), and the downstream planner or policy (completing a specified task) is one of the primary limitations in learning models of high dimensional states.  In other words, the learned model is encouraged to predict large portions of the state which may be irrelevant to the task at hand. Consider for example the task of picking up a pen from a cluttered desk. The standard training objective of the learned model would encourage it to equally weigh modeling the pen and all objects on the table, when given the downstream task, modeling the pen and adjacent objects precisely is clearly the most critical.

To that end, we propose goal-aware prediction (GAP), a framework for learning forward dynamics models that direct their capacity differently conditioned on the task, resulting in a model that is more accurate on trajectories
most relevant to the downstream task.
Specifically, we propose to learn a latent representation of not just the state, but both the state and goal, and to learn dynamics in this latent space. Furthermore, we can learn this latent space in a way that focuses primarily on parts of the state relative to achieving the goal,
namely by reconstructing the \emph{goal-state residual}
instead of the full state. We find that this modification combined with training via goal-relabeling \cite{her} allows us to learn expressive, task-conditioned dynamics models in an \emph{entirely self-supervised} manner. We observe that GAP learns dynamics that achieve significantly lower error on task relevant states, and as a result outperforms standard latent dynamics model learning
and self-supervised model-free reinforcement learning \cite{ashvinnairRIG} across a range of vision based control tasks. 
\section{Related Work}
\label{sec:related}

Recent years have seen impressive results from reinforcement learning \cite{Sutton1998} applied to challenging problems such as video games \cite{Mnih2015HumanlevelCT, OpenAI_dota}, Go \cite{AlphaGo}, and robotics \cite{DBLP:journals/corr/LevineFDA15, openai_dexterous, qtopt}. 
However, the dependence on large quantities of labeled data can limit the applicability of these methods in the real world. 
One approach is to leverage
self-supervision, where an agent only uses data that it can collect autonomously.

\textbf{Self-Supervised Reinforcement Learning:} Self-supervised reinforcement learning explores how RL can leverage data which the agent can collect autonomously to learn meaningful behaviors, without dependence on task specific reward labels, with promising results on tasks such as robotic grasping and object re-positioning \cite{DBLP:journals/corr/PintoG15, visualforesightebert, andyzengsynergy}. One approach to self-supervised RL has been combining goal-conditioned policy learning~\citep{Kaelbling93learningto, pmlr-v37-schaul15, Codevilla2017EndtoEndDV} with goal re-labeling \cite{her} or sampling goals \cite{ashvinnairRIG, Nair2019ContextualIG}. While there are numerous ways to leverage self-supervised data, ranging from learning distance metrics \cite{dpn, Hartikainen2019DynamicalDL}, generative models over the state space \cite{causalinfogan, Fang2019DynamicsLW,sorb, liu2020hallucinative,  nair2020hierarchical}, 
and representations \cite{Veerapaneni2019EntityAI}, one of the most heavily utilized techniques is learning the dynamics of the environment \cite{e2c,DBLP:journals/corr/FinnL16, Agrawal2016LearningTP}.

\textbf{Model-Based Reinforcement Learning:}
Learning a model of the dynamics of the environment and using it to complete tasks has been a well studied approach to solving reinforcement learning problems, either through planning with the model \cite{Deisenroth11pilco:a,e2c, McAllister2016ImprovingPW, Banijamali2017RobustLC, chua_handful, awarenessmodels_amos, DBLP:journals/corr/abs-1811-04551hafner, Nagabandi2019DeepDM} or optimizing a policy in the model \cite{Racanire2017ImaginationAugmentedAF, Ha2018WorldM, Kaiser2019ModelBasedRL, Lee2019StochasticLA, Janner2019WhenTT, Wang2019ExploringMP, Hafner2019DreamTC, Gregor2019ShapingBS, Byravan2019ImaginedVG}. 
Numerous works have explored how these methods might leverage deep neural networks to extend to high dimensional problem settings, such as images. One technique has been to learn large video prediction models \cite{DBLP:journals/corr/FinnL16,sv2p,ebertskip, visualforesightebert,paxton, savp, highfidelity_villegas, xietooluse}, however model under-fitting remains an issue for these approaches \cite{Dasari2019RoboNetLM}. Similarly, many works have explored learning low dimensional latent representations of high dimensional states \cite{e2c, Dosovitskiy2016LearningTA, Zhang2018SOLARDS, DBLP:journals/corr/abs-1811-04551hafner,causalinfogan, ichterlatent, wangvisualplan, Lee2019StochasticLA, deepmdp_gelada} and learning the dynamics in the latent space. Unlike these works, we aim to make the problem easier by encouraging the network to predict only task-relevant quantities, while also changing the objective, and hence the distribution of prediction errors, in a task-driven way. This allows the prediction problem to be more directly connected to the downstream use-case of task-driven planning. 

\textbf{Addressing Model Errors:}
Other works have also studied the problem of model error and exploitation. Approaches such as ensembles \cite{chua_handful, Thananjeyan2019SafetyAV} have been leveraged to measure uncertainty in model predictions. Similarly, \citet{Janner2019WhenTT} explore only leveraging the learned model over finite horizons where it has accurate predictions and \citet{DBLP:journals/corr/LevineFDA15} use local models. 
Exploration techniques can also be used to collect more data where the model is uncertain \cite{Pathak2017CuriosityDrivenEB}. 

Most similar to our proposed approach are techniques which explicitly change the models objective to optimize for performance on downstream tasks. \cite{ Schrittwieser2019MasteringAG, havens2020learning} explore only predicting future reward to learn a latent space in which they learn dynamics, 
\citet{Freeman2019LearningTP} learn a model with the objective of having a policy achieve high reward from training in it, and \citet{Amos2018DifferentiableMF, Srinivas2018UniversalPN} embed a model/planner inside a neural network.
Similarly, \citet{pmlr-v54-farahmand17a, d2019gradient, Lambert2020ObjectiveMI} explore how model training can be re-weighted using value functions, policy gradients, or expert trajectories to emphasize task performance. Unlike these works, which depend heavily on task-specific supervision, our approach can be learned on purely self-supervised data, and generalize to unseen tasks.

\section{Goal-Aware Prediction}
\label{sec:method}
\newtheorem{theorem}{Theorem}[section]

We consider a goal-conditioned RL problem setting (described next), for which we utilize a model-based reinforcement learning approach. The key insight of this work stems from the idea that the distribution of model errors greatly affects task performance and that, when faced with limited model capacity, we can control the distribution of errors to achieve better task performance.
We theoretically and empirically investigate this effect in Sections~\ref{sec:theory} and~\ref{sec:navigation} before describing our approach for skewing the distribution of model errors in Section~\ref{sec:gap}.

\subsection{Preliminaries}
\label{sec:prelim}

We formalize our problem setting as a goal-conditioned Markov decision process 
(MDP) defined by the tuple $(\mathcal{S}, \mathcal{A}, p, \mathcal{G}, \lambda)$ where $s \in \mathcal{S}$ is the state space, $a \in \mathcal{A}$ is the action space, $p(s_{t+1} | s_t, a_t)$
governs the environment dynamics, $p(s_0)$ corresponds to the initial state distribution,
$\mathcal{G} \subset \mathcal{S}$ represents the \emph{unknown} set of goal states which is a subset of possible states, 
and $\lambda$ is the discount factor. Note that this is simply a special  case of a Markov decision process, where we do not have access to extrinsic reward (i.e. it is self-supervised), and where we separate the state and goal for notational clarity.

We will assume that the agent has collected an unlabeled dataset $\mathcal{D}$ that consists of $N$ trajectories $[\tau_1, ..., \tau_N]$, and each trajectory $\tau$ consists of a sequence of state action pairs $[(s_0, a_0), (s_1, a_1), ..., (s_T)]$.  We will denote the estimated distance between two states as $C(s_t, s_g) = ||s_t-s_g||_2^2$, which may not accurately reflect the true distance, e.g. when states correspond to images. At test time, the agent is initialized at a start state $s_0 \sim p(s_0)$ with a goal state $s_g$ sampled at random from $\mathcal{G}$, and must minimize cost $\mathcal{C}(s_t, s_g)$. 
We assume that for any states $s_t, s_g$ we can measure $\mathcal{C}(s_t, s_g)$ as the distance between the states, for example in image space $\mathcal{C}$ would be pixel distance. Success is measured as reaching within some true distance of $s_g$.

In the model-based RL setting we consider here, the agent aims to solve the RL problem by learning a model of the dynamics $p_\theta(s_{t+1} | s_t, a_t)$ from experience, and using that model to plan a sequence of actions or optimize a policy.

\subsection{Understanding the Effect of Model Error on Task Performance}
\label{sec:theory}

A key challenge in model-based RL is that dynamics prediction error does not directly correspond to task performance. Specifically, for good task performance, certain model errors may be more costly than others, and if errors are simply distributed uniformly over dynamics predictions, errors in these critical areas may be exploited when selecting actions downstream. 
Intuitively, when optimizing actions for a given task, we would like our model to to give accurate predictions for actions 
that are important for completing the task, while the model likely does not need to be as accurate on trajectories that are completely unrelated to the task. In this section, we formalize this intuition.

Suppose the model is used by a policy to select from $N$ action sequences $a_{1:T}^i$, each with expected final cost $c^*_i = E_{p(s_{t+1} | s_t, a_t), a_{1:T}^i} [ \mathcal{C}(s_T, s_g)]$. 
Without loss of generality, let $c^*_1 \leq c^*_2 ... \leq c^*_N$, 
i.e. the order of action sequences is sorted by their cost under the true model, 
which is unknown to the agent. Denote $\hat{c}_i$ as the predicted final cost of action sequence $a_{1:T}^i$ under the learned model,
i.e. $\hat{c}_i = E_{p_\theta(s_{t+1} | s_t, a_t), a_{1:T}^i} [ \mathcal{C}(\hat{s_T}, s_g)]$. Moreover, we consider a policy that simply selects the action sequence with lowest cost under the model: $\hat{a} = \arg\min_{a_{1:T}^i} \hat{c}_i$. 
\textit{Let the policies behavior be $\epsilon$-optimal  if the cost of the selected action sequence $a_{1:T}^i$ has cost $c^*_i \leq c^*_1 + \epsilon$.}
Under this set-up, we now analyze how model error affects policy performance.

\begin{theorem}
\label{thrm}
The policy will remain $\epsilon$-optimal, that is, 
\begin{equation}
c^*_{i'} \leq c^*_1 + \epsilon ~~~~~ i' = \arg\min_i {\hat{c_i}}
\label{eq:4}
\end{equation}
if the following two conditions are met:
first, that the model prediction error on the \textbf{best} action sequence $a^1_{1:T}$ is bounded such that 
\vspace{-0.2cm}
\begin{equation}
|c^*_1 - \hat{c}_1 | < \epsilon
\label{eq:1}
\end{equation}
and second, that the errors of sub-optimal actions sequences $a_{1:T}^i$ are bounded by 
\begin{equation}
|c^*_i-\hat{c}_i| < (c^*_i - c^*_1) - \epsilon ~~~\forall i \mid c^*_i > c^*_1 + \epsilon
\label{eq:2}
\end{equation}
\end{theorem}

\begin{proof}
For the specified policy, violating $\epsilon$-optimality will only occur if
the cost of the best action sequence $a_{1:T}^1$ is overestimated or if the cost of a sub-optimal action sequence ($i \mid c^*_i > c^*_1 + \epsilon$) is underestimated. Thus, let us define the "worst case" cost predictions as the ones for which $c_1^*$ is most overestimated and $c_i^* ~~~\forall i \mid c^*_i > c^*_1 + \epsilon$ are most underestimated (while still satisfying Equations \ref{eq:1} and \ref{eq:2}).
Concretely we write the worst case cost estimates as 
$$\Tilde{c}_i := \min{\hat{c}_i} ~~~\forall i \mid c^*_i > c^*_1 + \epsilon$$
$$\Tilde{c}_1 := \max{\hat{c}_1}$$ 
s.t. Eq. \ref{eq:1} and \ref{eq:2} hold. We will now show that $\Tilde{c_1} < \Tilde{c_i} ~~~\forall i \mid c^*_i > c^*_1 + \epsilon$.
First, since $\Tilde{c_i}$ satisfies Eq. \ref{eq:2}, we have that $$\Tilde{c_i} > c^*_i - (c^*_i - c_1^*) + \epsilon$$ Similarly, since $\Tilde{c_1}$ satisfies Eq. \ref{eq:1}, we have that $$\Tilde{c_1} < c_1^* + \epsilon$$
Substituting, we see that 
\begin{equation}
\Tilde{c_i} > c^*_i - (c^*_i -c_1^*) + \epsilon = c_1^* + \epsilon > \Tilde{c_1} ~~~\forall i \mid c^*_i > c^*_1 + \epsilon
\label{eq:3}
\end{equation}
Hence even in the worst case, Equations \ref{eq:1} and \ref{eq:2} ensure that  $\hat{c_i} > \hat{c_1} ~~~\forall i \mid c^*_i > c^*_1 + \epsilon$, and thus no action sequence $i$ for which $c^*_i > c^*_1 + \epsilon$ will be selected and the policy will remain $\epsilon$-optimal. Note that action sequences besides $i=1$ for which $c^*_i \leq c^*_1 + \epsilon$ costs are unbounded, as it is ok for them to be significantly underestimated since selecting them still allows the policy to be $\epsilon$-optimal.
\end{proof}

\begin{figure*}[t]
\centerline{\includegraphics[width=1.8\columnwidth]{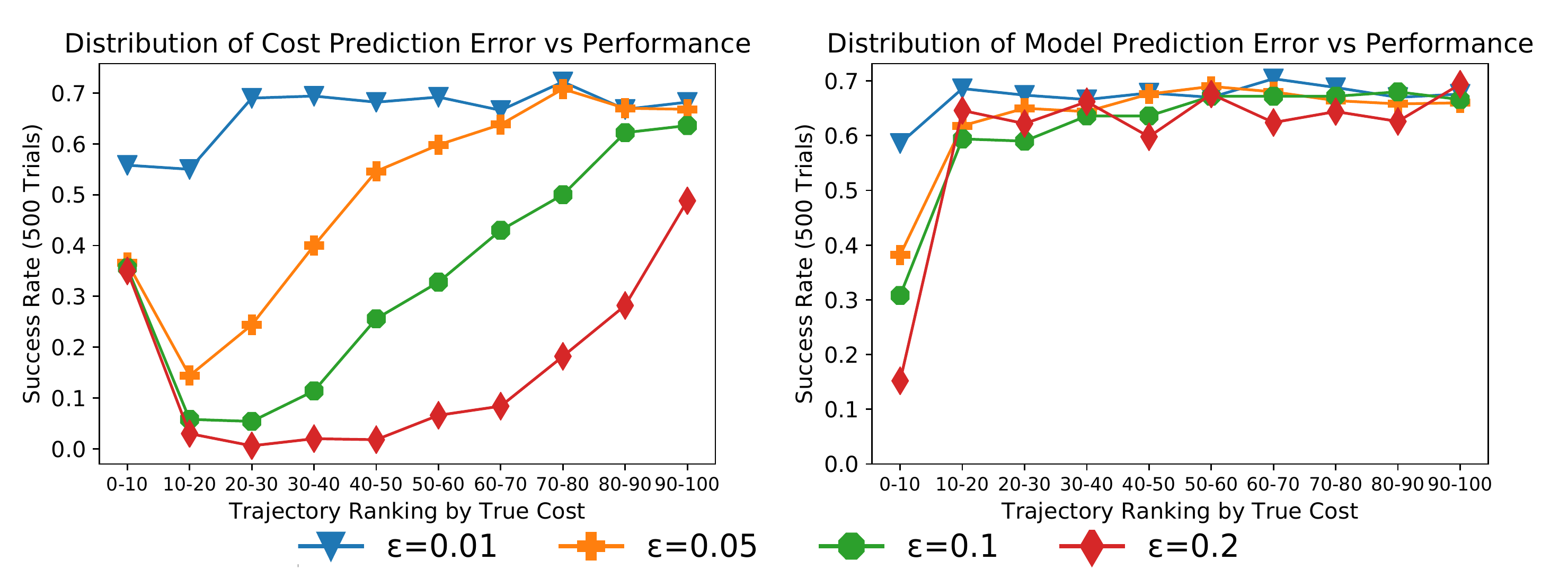}}
\vspace{-0.3cm}
\caption{\small{\textbf{Distribution of model errors vs. performance}: We validate how the distribution of model errors affects performance on a simple 2D navigation domain, by adding noise to cost predictions \textbf{(left)} or model predictions \textbf{(right)}. We add varying amounts of noise with magnitude up to $\varepsilon$ to the predictions of the 10 lowest true cost trajectories \textbf{(0-10)} to the 10 highest true cost trajectories \textbf{(90-100)}. We observe that adding noise to low true cost trajectories dramatically reduces performance, while adding noise to the high true cost trajectories has no nearly no impact on performance. 
}}
\vspace{-0.3cm}
\label{objm}
\end{figure*}

Theorem \ref{thrm} suggests that, for good task performance, model error must be low for good trajectories, and we can afford higher model error for trajectories with higher cost. That is, \textbf{the greater the trajectory cost, the more model error we can afford}.  Specifically, we see that the allowable error bound on cost of an action sequence from a learned model scales linearly with how far from optimal that action sequence is, in order to maintain the optimal policy for the downstream task. Note, that while Theorem \ref{thrm} relates cost prediction error (not explicitly dynamics prediction error) to planning performance, we can expect dynamics prediction error to relate to the resulting cost prediction error. We also verify this empirically in the next section.

\subsection{Verifying Theorem~\ref{thrm} Experimentally}
\label{sec:navigation}

We now verify the above analysis through a controlled study of how prediction error affects task performance. To do so, we will use the true model of an environment and true cost of an action sequence for planning, but will artificially add noise to the cost/model predictions to generate model error. 

Consider a 2 dimensional navigation task, where the agent is initialized at $s_0 = [0.5, 0.5]$
and is randomly assigned a goal $s_g \in [0,1]^2$.
Assume we have access to the underlying model of the environment,
and cost defined as $\mathcal{C}(s_t, s_g) = ||s_t-s_g||_2$. 
We can run the policy described in Section \ref{sec:theory}, specifically sampling $N=100$ action sequences,
and selecting the one with lowest predicted cost, where we consider 2 cases: (1) predicted cost is using the true model, but with noise $\alpha$ added to the true cost $\hat{c_i} = c_i^* + \alpha$ of some subset of action sequences, and (2) predicted cost is true cost, but with noise $\alpha$ added to the model predictions $s_{t+1} = \bar{s}_{t+1}+ \alpha$ where $\bar{s}_{t+1}\sim p(s_{t+1}|s_t,a_t)$ of some subset of action sequences. The first case relates directly to Theorem~\ref{thrm}, while the second case relates to what we can control when training a self-supervised dynamics model. When selectively adding noise, we will use uniform noise $\alpha \sim \mathcal{U}(-\varepsilon, \varepsilon)$. 
We specifically study the difference in task performance when adding noise $\alpha$ to model predictions for the first $10\%$ of trajectories with lowest true cost, the second $10\%$ lowest true cost trajectories, etc., up to the $10\%$ of trajectories with highest true cost. Here ``true cost'' refers to the cost of the action sequence under the true model and cost function without noise. For each noise augmented model we measure the task performance, specifically the success rate (reaching within 0.1 of the goal), over 500 random trials.

\begin{figure*}[t]
\centerline{\includegraphics[width=1.99\columnwidth]{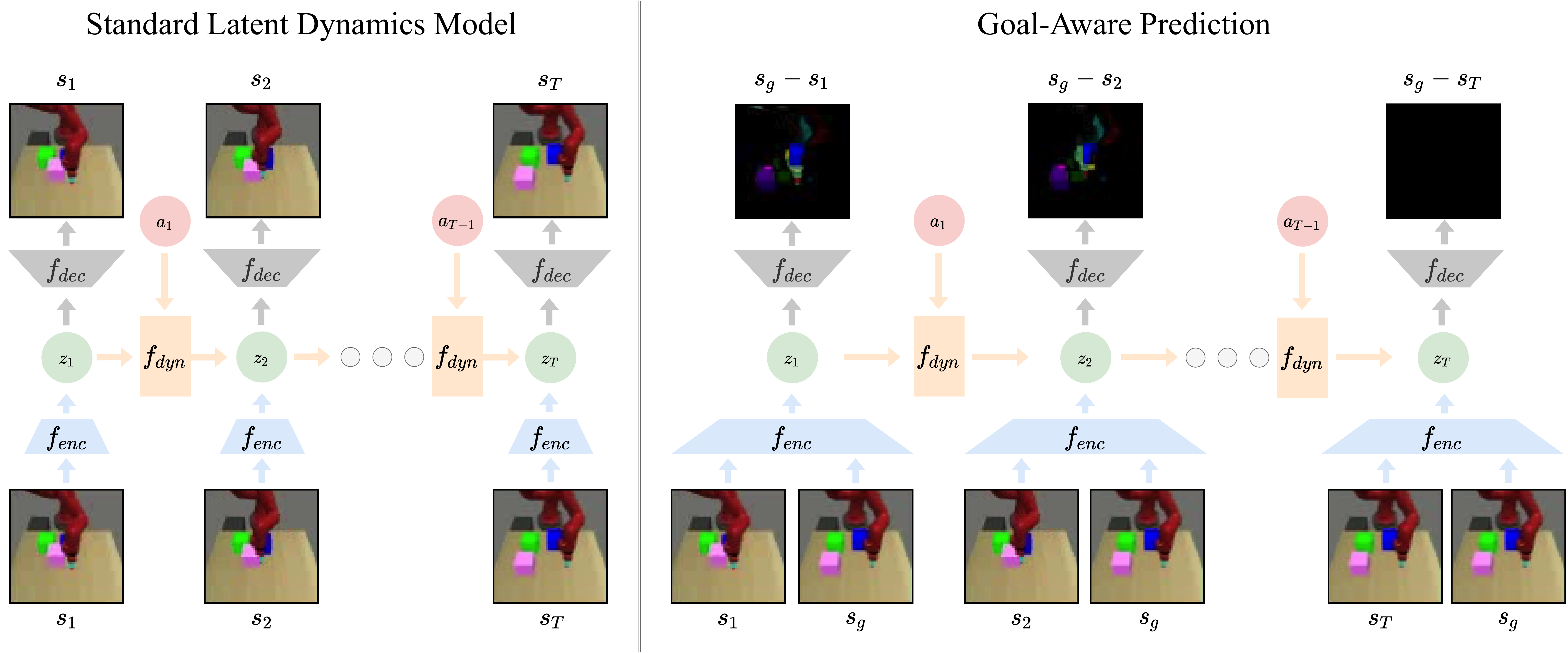}}
\vspace{-0.3cm}
\caption{\small{\textbf{Goal-Aware Prediction:} Compared to a standard latent dynamics model \textbf{(left)}, our proposed method, goal-aware prediction (GAP), \textbf{(right)} encodes both the current state $s_t$ and goal $s_g$ into a single latent space $z_t$. Samples from the distribution of $z_t$ are then used to reconstruct the residual between the current state and goal $s_g - s_t$. Simultaneously, we learn the forward dynamics in the latent space $z$, specifically, learning to predict $z_{t+1}$ from $z_t$ and $a_t$. Using this approach, we obtain 2 favorable properties: (1) the latent space only needs to capture components of the scene relevant to the goal, and (2) the prediction task becomes easier (the residual approaches 0) for states closer to the goal. }}
\vspace{-0.3cm}
\label{overview}
\end{figure*}

We see in Figure \ref{objm} that for multiple values of noise $\varepsilon$, when adding noise to the better (lower true cost) trajectories we see a significant drop in task performance, while when adding noise to the worse (higher true cost) trajectories task performance remains relatively unchanged (except for the case with very large $\varepsilon$). In particular, we notice that when adding noise to cost predictions, performance scales almost linearly as we add noise to worse trajectories.
Note there is one exception to this trend: if we add noise only to the top $10\%$ of trajectories, performance is not optimal, but reasonable because the best few trajectories will occasionally be assigned a lower cost under the noise model.

In the case of model error, we see a much steeper increase in performance, where adding model error to the best 10 trajectories significantly hurts performance, while adding to the others does not. This is because, in this environment, noise added to model predictions generally makes the cost of those predictions worse; so if no noise is added to the best trajectories, the best action sequence is still likely to be selected. The exact relationship between model prediction error and cost prediction error depends on the domain and task. But, we can see that in both cases in Figure~\ref{objm}, the conclusion from Theorem~\ref{thrm} holds true: accuracy on good action sequences matters much more than accuracy on bad action sequences.

\subsection{Redistributing Model Errors with Goal Aware Prediction}
\label{sec:gap}

Our analysis above suggests that distributing errors uniformly across action sequences will not lead to good task performance. Yet, standard model learning objectives will encourage just that. In this section, we aim to change our model learning approach in an aim to redistribute errors more appropriately.

Ideally, we would like to encourage the model to have more accurate predictions on the trajectories which are relevant to the task. However, actually identifying how relevant a trajectory is to a specific goal $s_g$ can be challenging.

One potential approach to doing this would be to re-weight the training loss of the model on transitions $p_\theta(s_{t+1}|s_t, a_t)$ \textit{inversely} by the
cost $\mathcal{C}(s_t, s_g)$, such that low cost trajectories are weighted more heavily. While this simple approach may work when the cost function  $\mathcal{C}(s_t, s_g)$ is accurate, the distance metric $\mathcal{C}(s_t, s_g)$ for high-dimensional states is often sparse and not particularly meaningful; when states are images, $\mathcal{C}$ amounts to $\ell_2$ distance in pixel space.

An alternative way of approaching this problem is, rather than focusing on how to re-weight the model's predictions, instead ask, ``what exactly should the model be predicting?'' If the downstream task involves using the model to plan to reach a goal state, then intuitively the model should only need to focus on predicting goal relevant parts of the scene. Moreover, if the model is trained to focus on parts of the scene relevant to the goal, it will naturally be biased towards higher accuracy in states relevant to the task, re-distributing model error favourably for downstream performance.

To that end, we propose goal-aware prediction (GAP) as a technique to re-distribute model error
by learning a model that,
in addition to the current state and action, $s_t$ and $a_t$, is conditioned on the goal state $s_g$, and instead of reconstructing the future state $s_{t+1}$, reconstructs the difference between the future state and the goal state, that is: $p_\theta( (s_g - s_{t+1}) | s_t, s_g, a_t)$. Critically, to train GAP effectively, we need action sequences that are relevant to the corresponding goal. To accomplish this, we can choose to set the goal state for a given action sequence as the final state of that trajectory, i.e. using hindsight relabeling \cite{her}.
Specifically, given a trajectory $[(s_1, a_1), (s_2, a_2), ..., (s_T)]$, the goal is assigned to be the last state in the trajectory $s_g = s_T$, and for all states $\{s_{t} | 1 \leq t \leq T-1\}  $, $p_\theta(s_t, s_g, a_{t})$ is trained to reconstruct the delta to the goal $s_g - s_{t+1}$.

Our proposed GAP method has two clear benefits over standard dynamics models. First, assuming that the agent is not in a highly dynamic scene with significant background distractor motion, by modeling the delta between $s_g$ and $s_t$, $p_\theta$ only needs to model components of the state which are relevant to the current goal. This is particularly important in high dimensional settings where there may be large components of the state which are irrelevant to the task, and need not be modeled.
Second, states $s_t$ that are temporally close to the goal state $s_g$ will have a smaller delta $s_g - s_t$, approaching zero along the trajectory until $s_t=s_g$. As a result, states closer to the goal will be easier to predict, biasing the model towards low error near states relevant to the goal. In light of our analysis of model error in the previous sections, we hypothesize that this model will lead to better downstream task performance compared to a standard model that distributes errors uniformly across trajectories.

\textit{When do we expect GAP to improve performance on downstream tasks?} We expect GAP to be most effective when the goals involve changing a subset $d < D$ state dimensions from the initial states $s_t \in \mathbb{R}_D$. Under these conditions, GAP only needs to predict the dynamics of the $d$ dimensions, while standard latent dynamics need to predicts all $D$ making GAP an easier problem. 
\section{Implementing Goal-Aware Prediction}
\label{sec:implem}

We implement GAP with a latent dynamics model, as shown in Figure \ref{overview}. Given a dataset of trajectories $[\tau_1, ..., \tau_N]$, we sample sequences of states $[(s_1, a_1), ..., (s_T)]$ where we re-label goal for the trajectory as $s_g = s_T$. 

The GAP model consists of three components, (1) an encoder $f_{enc}(z_t | s_t, s_g; \theta_{enc})$ that encodes the state $s_t$ and goal $s_g$ into a latent space $z_t$, (2) a decoder $f_{dec} (s_g - s_t | z_t;\theta_{dec}) $ that decodes samples from the latent distribution into $s_g-s_t$, and (3) a forward dynamics model in the latent space $f_{dyn} (z_{t+1} | z_t, a_t ; \theta_{dyn})$ which learns to predict the future latent distribution over $z_{t+1}$ from $z_t$ and action $a_t$. In our experiments we work in the setting where states are images, so $f_{enc}(z_t | s_t, s_g)$ and $f_{dec} (s_g - s_t | z_t) $ are convolutional neural networks, and $f_{dyn} (z_{t+1} | z_t, a_t)$ is a fully-connected network. The full set of parameters $\theta = \{\theta_{enc}, \theta_{dec}, \theta_{dyn}\}$ are jointly optimized. Exact architecture and training details for all modules can be found in the supplement.
Following prior works~\cite{Finn2016UnsupervisedLF, DBLP:journals/corr/FinnL16, awarenessmodels_amos}, we train for multi-step prediction. More specifically, given $s_t, a_{t:t+H}, s_g$, the model is trained to reconstruct $(s_g - s_t), ..., (s_g - s_{t+H})$, shown in Figure~\ref{overview}. 

\textbf{Data Collection and Model Training:}
In our \textit{self-supervised} setting, data collection simply corresponds to rolling out a random exploration policy in the environment. Specifically, we sample uniformly from the agent's action space, and collect 2000 episodes, each of length 50, for a total of 100,000 frames of data.  

During training, sub-trajectories of length 30 time steps are sampled from the data set, with the last timestep labeled as the goal $s_g=s_{30}$. Depending on the current value of $H$, loss is computed over $H$ step predictions starting from states $s_{t:(t+H)}$. 
We use a curriculum when training all models, where $H$ starts at 0, and is incremented by 1 every 50,000 training iterations. All models are trained to convergence, 
for about $300,000$ iterations on the same dataset.

\textbf{Planning with GAP:}
For all trained models, when given a new goal at test time $s_g$, we plan using model predictive control (MPC) in the latent space of the model. 
Specifically, both the current state $s_t$ and $s_g$ are encoded into their respective latent spaces $z_t$ and $z_g$ (Algorithm 1, Line 3). 

\begin{algorithm}[H]
\caption{$\text{Latent MPC}(f_{enc}, f_{dyn}, s_t, s_g)$}
\begin{algorithmic}[1]
\small{
\STATE Let $D=1000, D^*=10, H=15$
\STATE Receive current state $s_t$ and goal state $s_g$
\STATE Encode $z_t \sim f_{enc}(s_t, s_g), z_g \sim f_{enc}(s_g, s_g)$
\STATE Initialize $N(\mu, \sigma^2) = \mathcal{N}(0, 1)$
\STATE Let cost function $C(z_i,z_j) =||z_i-z_j||_2^2 $
\WHILE{\text{iterations} $\leq 3$}
    \STATE{$a_{t:H}^1, ..., a_{t:H}^D \sim N(\mu, \sigma^2)$}
    \STATE{$z_{t+1:t+H}^1, ..., z_{t+1:t+H}^D \sim f_{dyn}(z_t, a_{t:H}^1, ..., a_{t:H}^D)$}
    \STATE{$\hat{c}^1, ..., \hat{c}^D = [\sum_{h=1}^H \mathcal{C}(z_{t+h}^1, z_g), ..., \sum_{h=1}^H \mathcal{C}(z_{t+h}^D, z_g)] $}
    \STATE{$a_{sorted} = Sort([a_{t:H}^1, ..., a_{t:H}^D])$} by $\hat{c}$
    \STATE{Refit $\mu, \sigma^2$ to $a_{sorted}[1:D^*]$}
\ENDWHILE
\STATE Return $\hat{c}_{sorted}[1]$, $a_{sorted}[1]$}
\end{algorithmic}
\end{algorithm}
\vspace{-0.5cm}
Then using the model $f_{dyn} (z_{t+1} | z_t, a_t)$, the agent plans a sequence of $H$ actions to minimize cost $\sum_{h =0}^H ||z_g-z_{t+h}||_2^2$ (Algorithm 1, Lines 4-11). Following prior works \cite{DBLP:journals/corr/FinnL16, DBLP:journals/corr/abs-1811-04551hafner}, we use the cross-entropy method \cite{cem} as the planning optimizer. 
Finally, the best sequence of actions is returned and executed in the environment (Algorithm 1, Line 13).

While executing the plan, our model re-plans every $H$ timesteps. That is, it starts at state $s_t$, uses Latent MPC (Algorithm 1) to first plan a sequence of $H$ actions, executes them in the environment resulting in a state $s_{t+H}$, then re-plans an additional $H$ actions, and executes them resulting in a final state $s_T$. Success is computed based the difference between $s_T$ and $s_g$.

\section{Experiments}
\label{sec:experiments}

\begin{figure*}[t]
\centerline{\includegraphics[width=1.99\columnwidth]{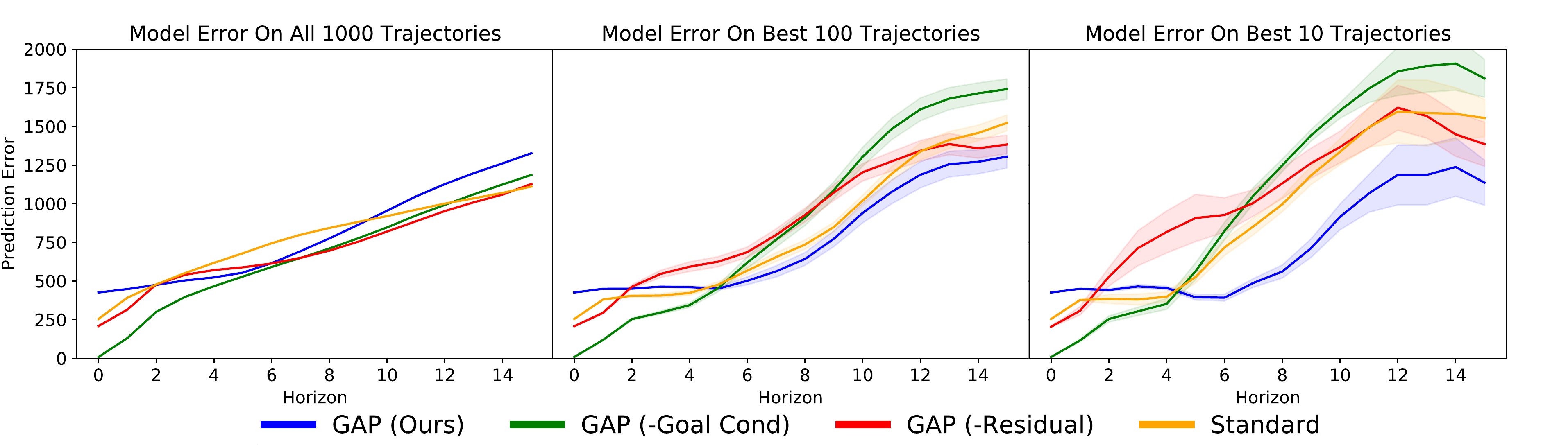}}
\vspace{-0.3cm}
\caption{\small{\textbf{Distribution of Model Errors:} We examine the distribution of model prediction errors of GAP compared to prior methods over 1000 random action sequences, evaluated on the ``Task 1'' domain. The y-axis are corresponds to model mean-squared error (with standard error bars), and the x-axis corresponds to number of time steps predicted forward.
Naturally, we observe that model error increases as the prediction horizon increases, for all approaches. However, although all approaches have a similar error over all 1000 action sequences (left), GAP achieves significantly lower error on the \textit{best} 10 trajectories (right). This suggests that changing the model objective through predicting the goal-state residual leads to more accurate predictions in areas that matter in downstream tasks.}}
\vspace{-0.3cm}
\label{modelerr}
\end{figure*}

In our experiments, we investigate three primary questions 

\textbf{(1)} Does using our proposed technique for goal-aware prediction (GAP) re-distribute model error such that predictions are more accurate on good trajectories?

\textbf{(2)} Does re-distributing model errors using GAP result in better performance in downstream tasks?

\textbf{(3)} Can GAP be combined with large video prediction models to scale to the complexity of real world images?

We design our experimental set-up with these questions in mind in Section \ref{sec:exp_overview}, then examine each of the questions in Sections \ref{model_error}, \ref{task_perf}, and \ref{sec:robonet} respectively. 

\begin{figure}[b]
\vspace{-0.3cm}
\centerline{\includegraphics[width=0.9\columnwidth]{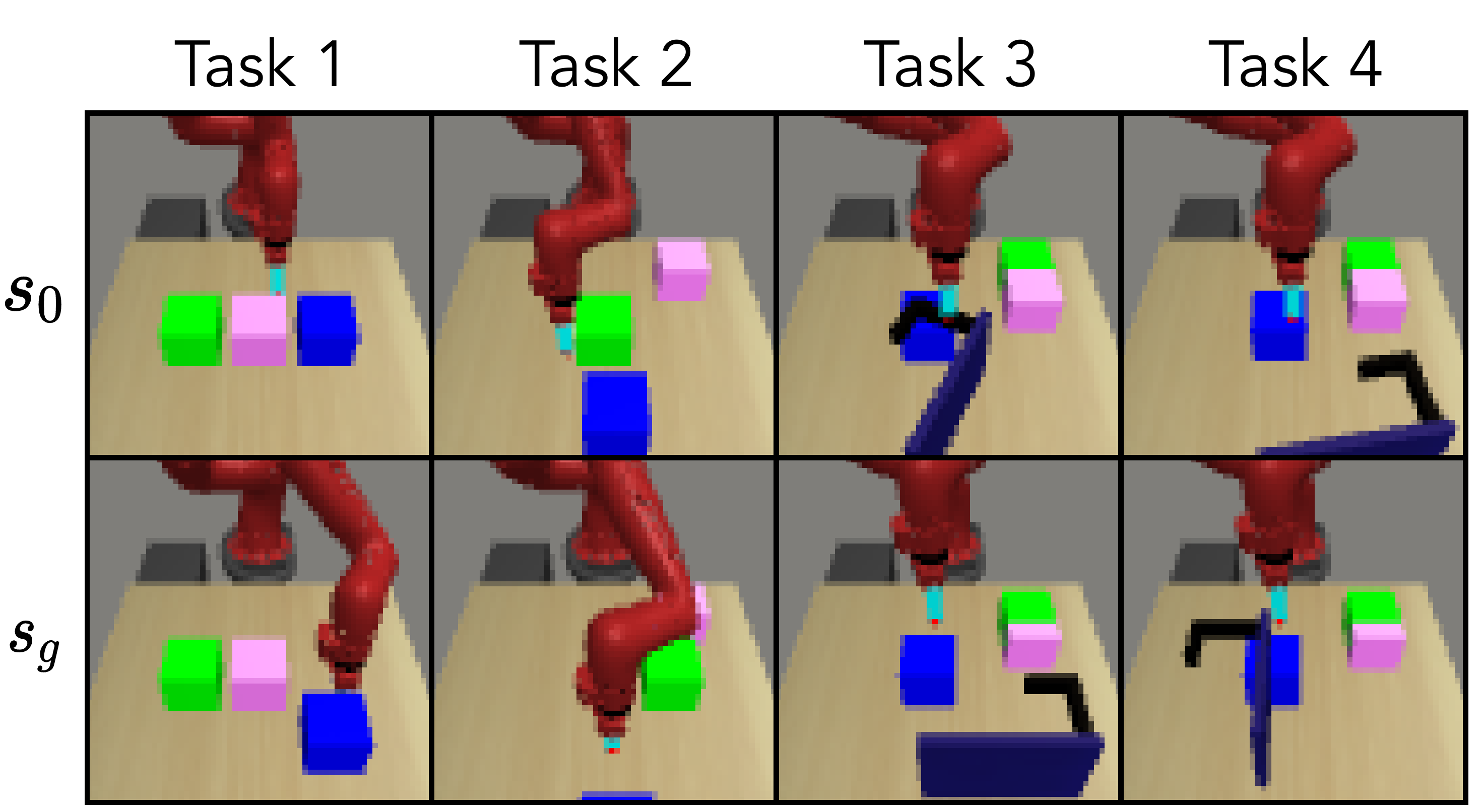}}
\vspace{-0.3cm}
\caption{\small{\textbf{Evaluation Tasks:} Sample initial \& goal states for each of the simulated manipulation tasks. Tasks involve manipulating blocks or a door, with the task specified by a goal image.}}
\vspace{-0.3cm}
\label{tasks}
\end{figure}

\subsection{Experimental Domains and Comparisons}
\label{sec:exp_overview}

\textbf{Experimental Domains}: Our primary experimental domain is a simulated tabletop manipulation task built off of the Meta-World  suite of environments \cite{Yu2019MetaWorldAB}.
Specifically, it consists of a simulated Sawyer robot, and 3 blocks on a tabletop. In the \textit{self-supervised} data collection phase, the agent executes a random policy for 2,000 episodes, collecting 100,000 frames worth of data. Then, after learning a model, the agent is tested on 4 previously unseen tasks, where the task is specified by a goal image.

\textbf{Task 1} consists of pushing the green, pink, or blue block to a goal position,
while the more challenging \textbf{Task 2} requires the robot to push 2 blocks to each of their respective goal positions (see Figure \ref{tasks}). Task success is defined as being within 0.1
of the goal positions. \textbf{Task 3 and 4} involve closing and opening the door respectively with distractor objects on the table, where success is defined as being within $\pi / 6$ radians of the goal position. The agent receives 64$\times$64 RGB camera observations of the tabletop. We also study model error  on real robot data from the BAIR Robot Dataset \cite{ebertskip} and RoboNet dataset \cite{Dasari2019RoboNetLM}
in Section \ref{sec:robonet}.

\textbf{Comparisons:} We compare to several model variants in our experiments.
\textbf{GAP} is our approach of learning dynamics in a latent space conditioned on the current state and goal, and reconstructing the residual between the current state and goal state, as described in Section \ref{sec:gap}.
\textbf{GAP (-Goal Cond)} is an ablation of GAP that does not use goal conditioning. Instead of conditioning on the goal and predicting the residual to the goal, it is conditioned on the initial state, and predicts the residual to the initial state. This is representative of prior works (e.g. \citet{Nagabandi2019DeepDM}) that predict residuals for model-based RL. 
\textbf{GAP (-Residual)} is another ablation of GAP that is conditioned on the goal but maintains the standard reconstruction objective instead of the residual. This is similar to prior work on goal conditioned video prediction \cite{rybkin2020goalconditioned}.
\textbf{Standard} refers to a standard latent dynamics model, representative of approaches such as PlaNet \cite{DBLP:journals/corr/abs-1811-04551hafner}, but without reward prediction since we are in the self-supervised setting.

When studying task performance, we also compare to two alternative self-supervised reinforcement learning approaches. First, we compare to an
\textbf{Inverse Model}, which is a latent dynamics model where the latent space is learned via an action prediction loss (instead of image reconstruction),
as done in \citet{Pathak2017CuriosityDrivenEB}. Second, we compare to a model-free approach: reinforcement learning with imagined goals (\textbf{RIG}) \cite{ashvinnairRIG}, where we train a VAE on the same pre-collected dataset as the other models, then train a policy in the latent space of the VAE to reach goals sampled from the VAE. 
Further implementation details can be found in the supplement.

\subsection{Experiment 1: Does GAP Favorably Redistribute Model Error?}
\label{model_error}

\begin{figure}[t]
\centerline{\includegraphics[width=0.99\columnwidth]{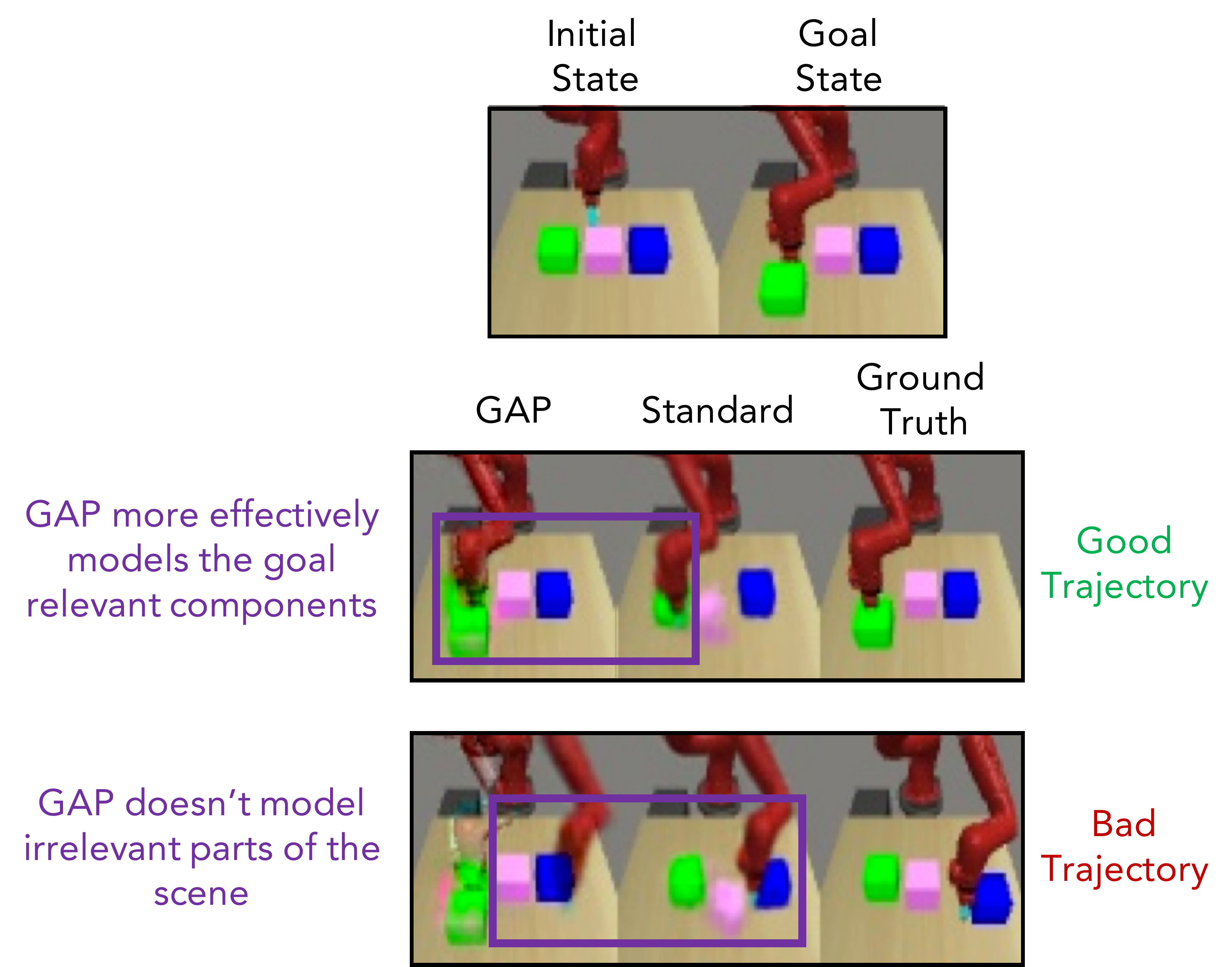}}
\vspace{-0.3cm}
\caption{\small{\textbf{GAP Predictions on Good/Bad Trajectories.} Here we show qualitatively how GAP focuses on the task relevant parts of the scene. Note, for GAP predictions we add back the goal image to the predicted goal-state residual. Given the task specified by pushing the green block \textbf{(top)}, consider a good action sequence \textbf{(middle)} and bad action sequence \textbf{(bottom)}. On the good action sequence GAP more effectively models the goal relevant parts of the scene (the green block) than the standard model. Additionally, on the bad trajectory, GAP ignores the irrelevant objects and does not model their dynamics at all, while the standard model does.  }}
\vspace{-0.3cm}
\label{qualsim}
\end{figure}

\begin{figure}[b]
\vspace{-0.4cm}
\centerline{\includegraphics[width=0.99\columnwidth]{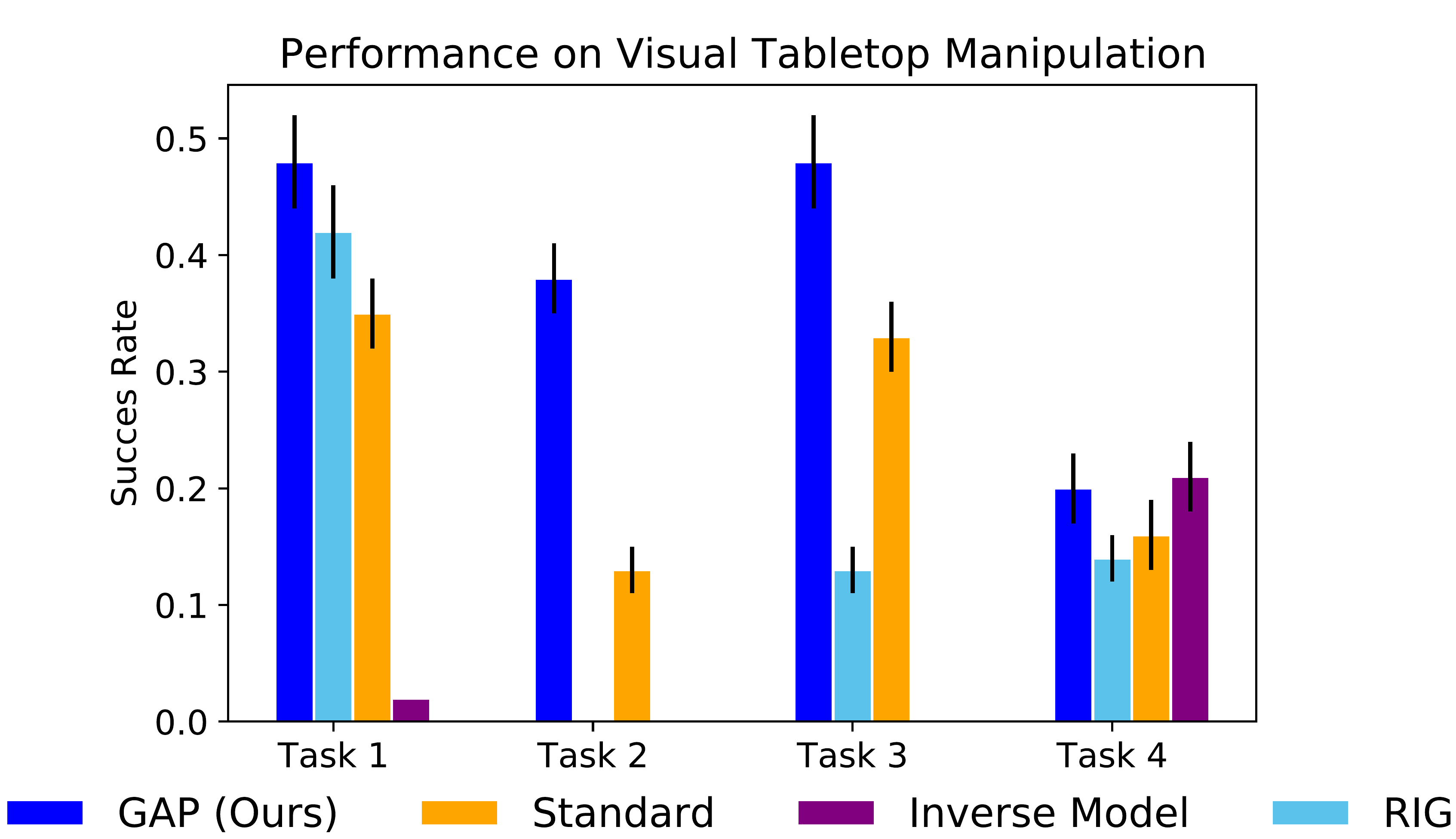}}
\vspace{-0.3cm}
\caption{\small{\textbf{Success rate on tabletop manipulation.} 
 On the tasks proposed in Section \ref{sec:exp_overview}, we find that GAP outperforms the comparisons. Specifically on the harder 2 block manipulation task, GAP has a significantly higher success rate.
  } }
\vspace{-0.2cm}
\label{task_perf_comp}
\end{figure}

\begin{figure}[t]
\centerline{\includegraphics[width=0.93\columnwidth]{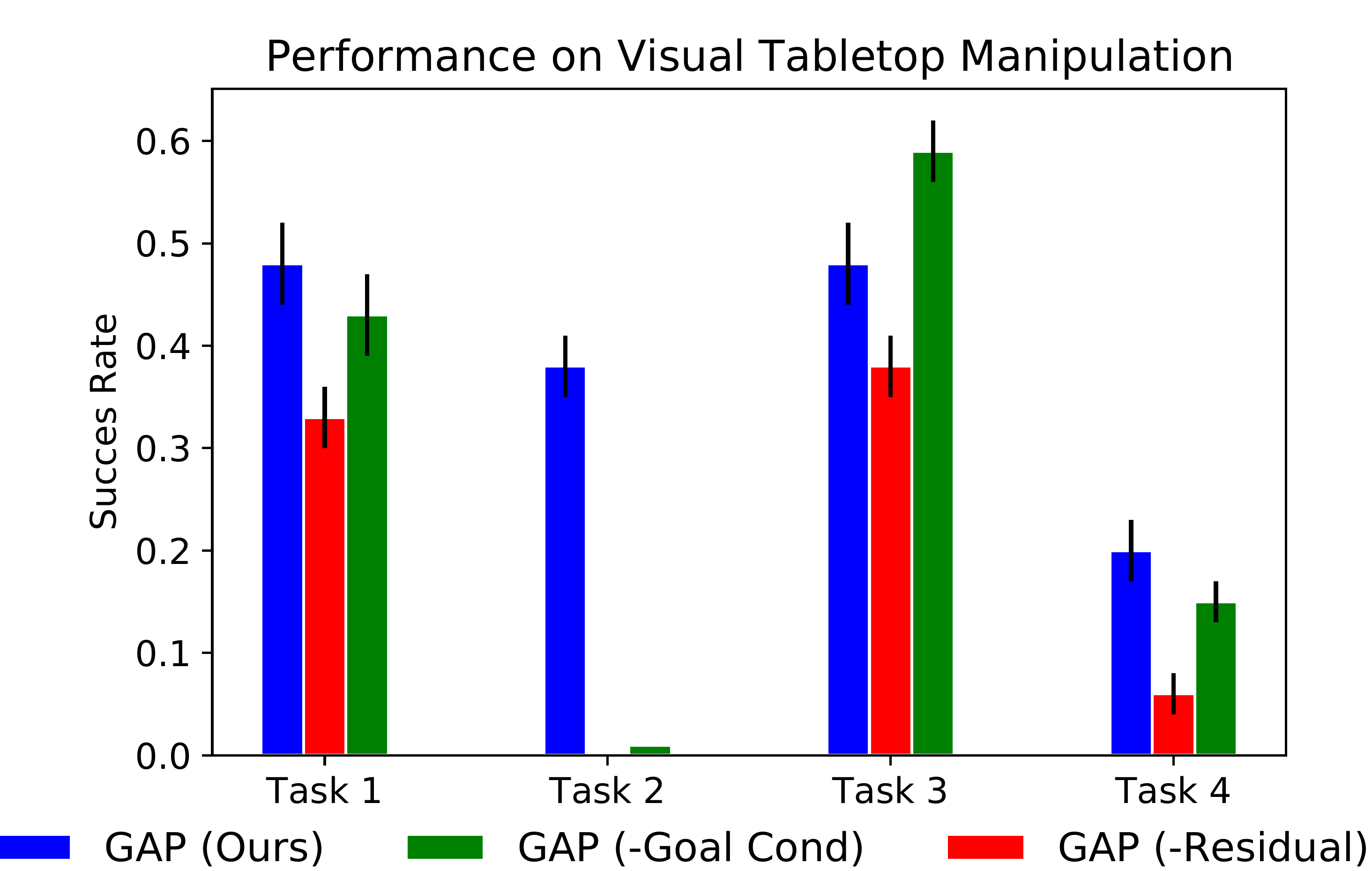}}
\vspace{-0.3cm}
\caption{\small{\textbf{Success rate on tabletop manipulation (ablation).} 
 We compare the success rate of GAP to ablations on the task described in Section \ref{sec:exp_overview}. We find that in all tasks except Task 3 both goal conditioning and residual are important for good performance. }
  } 
\vspace{-0.6cm}
\label{task_perf_ablation}
\end{figure}

In our first set of experiments, we study how GAP affects the distribution of model errors, and if it leads to lower model error on task relevant trajectories.
We sample 1000 random action sequences of length 15 in the Task 1 domain. We compute the true next states $s_{1:H}^1, ..., s_{1:H}^{1000}$ and costs $c^1, ..., c^{1000}$ for each action sequence by feeding it through the true simulation environment. We then get the predicted next states from our learned models, including GAP as well the comparisons outlined above. We then examine the model error of each approach, and how it changes when looking at all trajectories, versus the lowest cost trajectories.

We present our analysis in Figure \ref{modelerr}. We specifically look at the model error on all 1000 action sequences, the top 100 action sequences, and the top 10 action sequences. First, we observe that model error increases with the prediction horizon, which is expected due to compounding model error. More interestingly, however, we observe that while our proposed GAP approach has the highest error averaged across all 1000 action sequences, it has by far the lowest error on the top 10. This suggests that the goal conditioned prediction of the goal-state residual indeed encourages low model error in the relevant parts of the state space. Furthermore, we see that the conditioning on and reconstructing the difference to the \textit{actual goal} is in fact critical,
as the ablation GAP (-Goal Cond) which instead is conditioned on and predicts the residual to the first frame actually gets worse error on the lowest cost trajectories.

This indicates that GAP successfully re-distributes error  such that it has the most accurate predictions on low-cost trajectories. We also observe this qualitatively in Figure \ref{qualsim}. For a given initial state and goal state from Task 1,  GAP effectively models the target object (the green block) on a good action sequence that reaches the goal, while the standard model struggles. On a poor action sequence that hits the non-target blocks,  the Standard approach models them, while \textbf{GAP does not model interaction with these blocks at all,} suggesting that GAP does not model irrelevant parts of the scene. 
In the next section, we examine if this error redistribution translates to better task performance.

\begin{figure*}[t]
\centerline{\includegraphics[width=1.99\columnwidth]{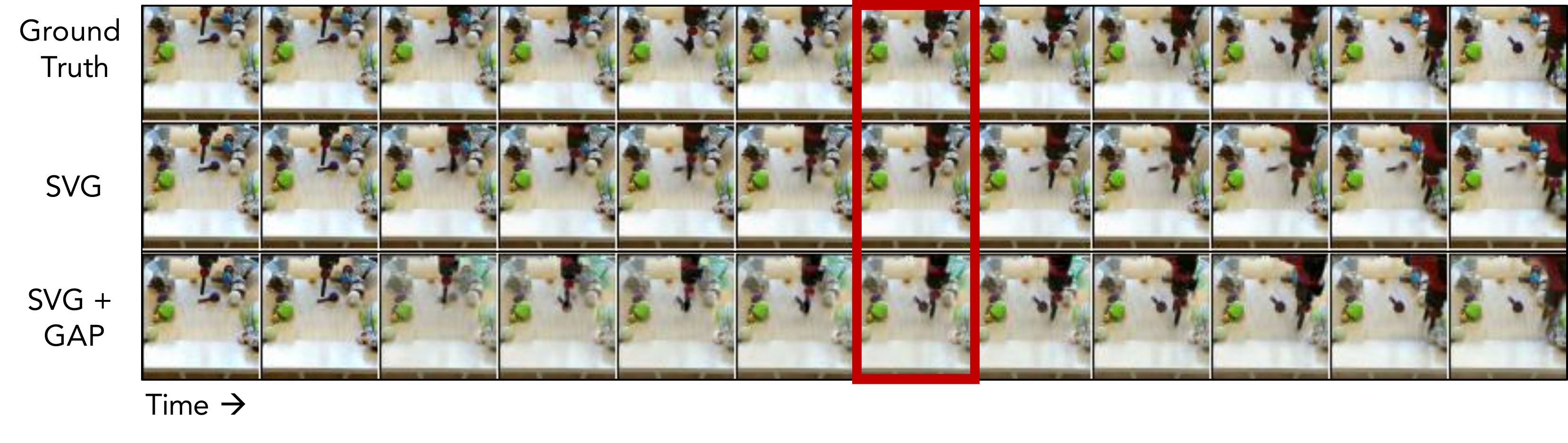}}
\vspace{-0.3cm}
\caption{\small{\textbf{GAP+SVG Video Prediction (BAIR Robot Dataset):} Here we present qualitative examples of action-conditioned SVG with and without GAP on the BAIR robot dataset, predicting on goal-reaching trajectories. Note, in the GAP predictions the goal is added back to the predicted goal-state residual. In this case the goal is the rightmost frame. We see that GAP is able to more accurately predict the objects relevant to the goal, for example the small spoon highlighted in the red box.}   }
\vspace{-0.2cm}
\label{qual_bair}
\end{figure*}

\subsection{Experiment 2: Does GAP Lead to Better Downstream Task Performance?}
\label{task_perf}

To study downstream task performance, we test on the tabletop manipulation tasks described in Section \ref{sec:exp_overview}. We perform planning over 30 timesteps with the learned models as described in Section \ref{sec:method},
and report the final success rate of each task over 200 trials in Figure \ref{task_perf_comp}. We see that in all tasks GAP outperforms the comparisons, especially in the most challenging 2 block manipulation task in Task 2 (where precise modeling of the relevant objects is especially important).

We make a similar comparison, but to the ablations of GAP, in Figure \ref{task_perf_ablation}. Again we see that GAP is performant, and in all tasks except Task 3 both goal conditioning and residual are important for good performance. Interestingly, we observe that GAP (-Goal Cond) is competitive on the door manipulation tasks.

Hence, we can conclude that GAP not only enables lower model error in task relevant states, but by doing so, also achieves a  10-20\% absolute performance improvement over baselines on 3 out of 4 tasks.

\subsection{Experiment 3: Does GAP scale to real, cluttered visual scenes?}
\label{sec:robonet}

Lastly, we study whether our proposed GAP method extends to real, cluttered visual scenes. To do so we combine it with an action-conditioned version of the video generation model, SVG \cite{Denton2018StochasticVG}. Specifically, we condition the SVG encoder on the goal, and the current goal-state residual, and predict the next goal-state residual.

We compare the prediction error of SVG to SVG+GAP on goal reaching trajectories (Figure \ref{robonet_modelerr}) from real robot datasets, namely the BAIR Robot Dataset \cite{ebertskip} and the RoboNet Dataset \cite{Dasari2019RoboNetLM}. We see that action-conditioned SVG combined with GAP as well as the ablation without residual prediction both have lower prediction error on goal reaching trajectories than standard action-conditioned SVG.

Qualitatively, we also observe that SVG+GAP is able to more effectively capture goal relevant components, as shown in Figure \ref{qual_bair}. We see that GAP is able to capture the motion of the small spoon, while correctly modeling the dynamics of the arm, while SVG ignores the spoon.

As a result, we conclude that GAP can effectively be combined with large video prediction models, and scaled to challenging real visual scenes.

\begin{figure}[t]
\centerline{\includegraphics[width=0.99\columnwidth]{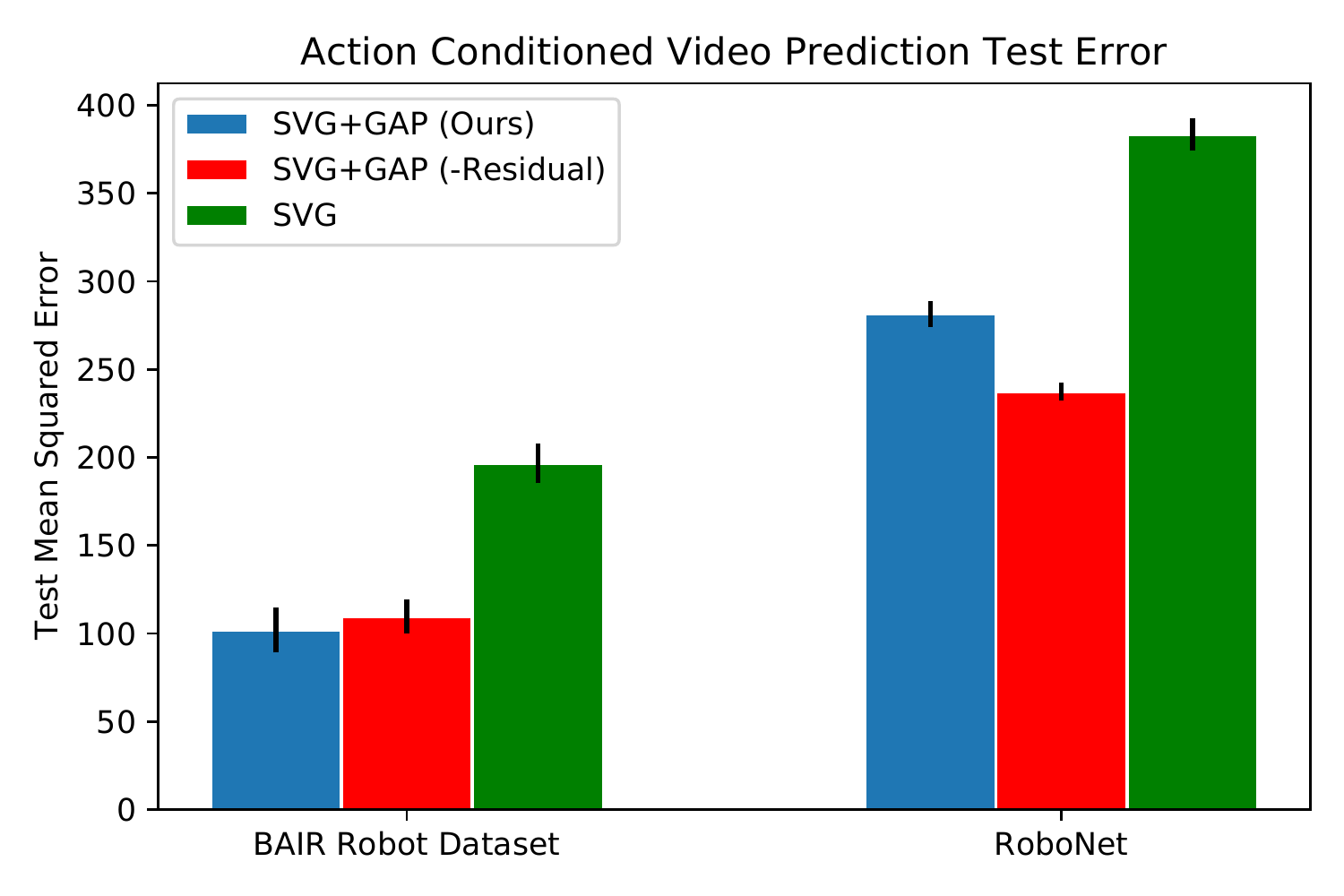}}
\vspace{-0.3cm}
\caption{\small{\textbf{Model Errors (Real Robot Data):} We examine the model error of SVG combined with GAP on \textbf{unseen, goal-reaching trajectories} from two real robot datasets (the BAIR Dataset \cite{ebertskip} and the RoboNet Dataset \cite{Dasari2019RoboNetLM}). We see that action-conditioned SVG combined with GAP has lower prediction error on the goal reaching trajectories than standard action-conditioned SVG. We observe that the GAP ablation which also conditions on the goals, but predicts residuals is equally effective in this setting. 
}}
\vspace{-0.2cm}
\label{robonet_modelerr}
\end{figure}

\section{Discussion and Limitations}

In this paper, we studied the role of model error in task performance. Motivated by our analysis, we proposed goal-aware prediction, a self-supervised framework for learning dynamics that are conditioned on a goal and structured in a way that favorably re-distributes model error to be low in goal-relevant states. In visual control domains, we verified that GAP (1) enables lower model error on task relevant states, (2) improves downstream task performance, and (3) scales to real, cluttered visual scenes.

While GAP demonstrated significant gains, multiple limitations and open questions remain. Our theoretical analysis suggests that we should re-distribute model errors according to their planning cost. While GAP provides one way to do that in a \emph{self-supervised} manner, there are likely many other approaches that can be informed by our analysis, including approaches that leverage human supervision. For example, we anticipate that GAP-based models may be less suitable for environments with dynamic distractors such as changing lighting conditions and moving distractor objects, since GAP would be still encouraged to model these events. To effectively solve this case, an agent would likely require human supervision to indicate the axes of variation that are relevant to the goal. Incorporating such supervision is outside the scope of this work, but an exciting avenue for future investigation.

Additionally, while in this work we found GAP to work well with goals selected at the end of sampled trajectories, there may be more effective ways to sample goals.
Studying the relationship between how exactly goals are sampled and learning performance, as well as how best to sample and re-label goals is an exciting direction for future work.

\section*{Acknowledgments}

We would like to thank Ashwin Balakrishna, Oleh Rybkin, and members of the IRIS lab for many valuable discussions. This work was supported in part by Schmidt Futures and an NSF graduate fellowship. Chelsea Finn is a CIFAR Fellow in the Learning in Machines \& Brains program.

\bibliography{references}

\appendix
\section{Method Implementation Details}

In this section we go over implementation details for our method as well as our comparisons. 

\subsection{Architecture Details}

\textbf{Block/Door Domain:} All comparisons leverage a nearly identical architecture, and are trained on an Nvidia 2080 RTX.
In the block pushing domain input observations are [64,64, 6] in the case of our model (GAP), as well as the ablations, and [64,64, 3] in the case of Standard.

All use an encoder $f_{enc}$ with convolutional layers (channels, kernel size, stride): [(32, 4, 2), (32, 3,1), (64, 4, 2), (64, 3,1), (128, 4, 2), (128, 3,1), (256, 4, 2), (256, 3,1)] followed by fully connected layers of size [512, 2 $\times L$] where $L$ is the size of the latent space (mean and variance). All layers except the final are followed by ReLU activation.

The decoder $f_{dec}$ takes a sample from the latent space of size $L$, then is fed through fully connected layers [128, 128, 128], followed by de-convolutional layers (channels, kernel size, stride): [(128, 5, 2), (64, 5, 2), (32, 6, 2), (3, 6,2)]. All layers except the final are followed bu ReLu activation, except the last layer which is a Sigmoid in the case of Standard, and GAP (-Residual).

For all models the dynamics model $f_{dyn}$ are a fully connected network with layers [128, 128, 128, $L$], followed by ReLU activation except the final layer.

The inverse model baseline utilizes the same $f_{enc}$ and $f_{dyn}$ as above, but $f_{dec}$ is instead a fully connected network of size [128, 128, action size] where action size is 4 (corresponding to delta x,y, z motion and gripper control). All layers except the final are followed by ReLU activation. 

Lastly, the RIG \cite{ashvinnairRIG} baseline uses a VAE with identical $f_{enc}$ and $f_{dec}$ to the standard approach above, except learns a policy in the latent space. The policy architecture used is the default SAC \cite{Haarnoja2018SoftAO} from RLkit, namely 2 layer MLPs of size 256. 

\textbf{SVG+GAP:} In all SVG \cite{Denton2018StochasticVG} based experiments on real robot data, the architecture used is identical to the SVG architecture as described in official repo\footnote{\url{https://github.com/edenton/svg}} with the VGG encoder/decoder. All BAIR dataset experiments take as input sequences of 2 frames and predict 10 frames, while all RoboNet experiments take as input 2 frames and predict 20 frames. The latent dimension is 64, and the encoder output dimension is 128. All models are trained with batch size 32.

\subsection{Training Details}

\textbf{Block/Door Domain:} We collect a dataset of 2,000 trajectories, each 50 timesteps with a random policy. All models are trained on this dataset to convergence for roughly 300,000 iterations. All models are trained with learning rate of 1e-4, and batch size 32. 

The RIG baseline is trained using the default SAC example parameters in RLkit, for an additional 3 million steps.

\textbf{BAIR Robot Dataset:} We train on the BAIR Robot Dataset \cite{ebertskip} as done in the original SVG paper, except with action conditioning. 

\textbf{RoboNet:} We train on the subset of the RoboNet dataset which considers only the sawyer arm and the front facing camera view, and use a random 80/20 train test split.

\subsection{Task/Evaluation Details}

\textbf{Tasks}. All tasks are defined in a Mujoco simulation built off the Meta-World environments \cite{Yu2019MetaWorldAB}. In Task 1, the agent must push a single block to a target position, as specified by a goal image. The task involves either pushing the pink, green, or blue block. Success if defined as being within 0.08 of the goal position. In Task 2 the agent must push 2 blocks, specifically the green and blue block to their respective goal positions, again indicated by a goal image. Success is determined as both blocks being within 0.1 of their respective goal positions. In Tasks 3 and 4 the agent must close or open a door, as specified by a goal image, where success is definged as being within $\pi / 6$ radians of the goal angle. 

\textbf{Evaluation.} During all control experiments, evaluation is done using model predictive control with the latent space dynamics model. Specifically, we do latent MPC as described in Algorithm 1, specifically by planning 15 actions, executing them in the environment, then planning 15 more actions and executing them. Each stage of planning uses the cross entropy method, specifically sampling 1000 action sequences, sorting them by the mean latent distance cost to the goal, refitting to the top 10, and repeating 3 times, before selecting the total lowest cost action.

\section{Additional Results}

In Figure~\ref{qual_bair_extra} we present additional examples of combining SVG with GAP on the BAIR dataset. Again we consider the goal image to be the last state in the trajectory. We see that using GAP leads to more effectively capturing task relevant objects, highlighted in red.

\begin{figure*}[t]
\centerline{\includegraphics[width=1.99\columnwidth]{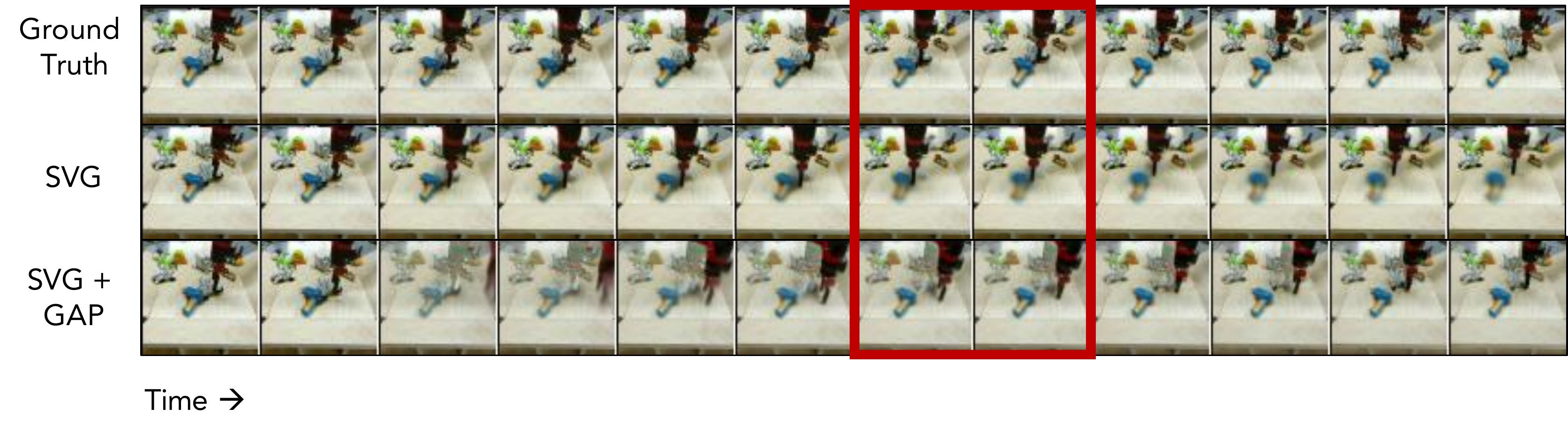}}
\centerline{\includegraphics[width=1.99\columnwidth]{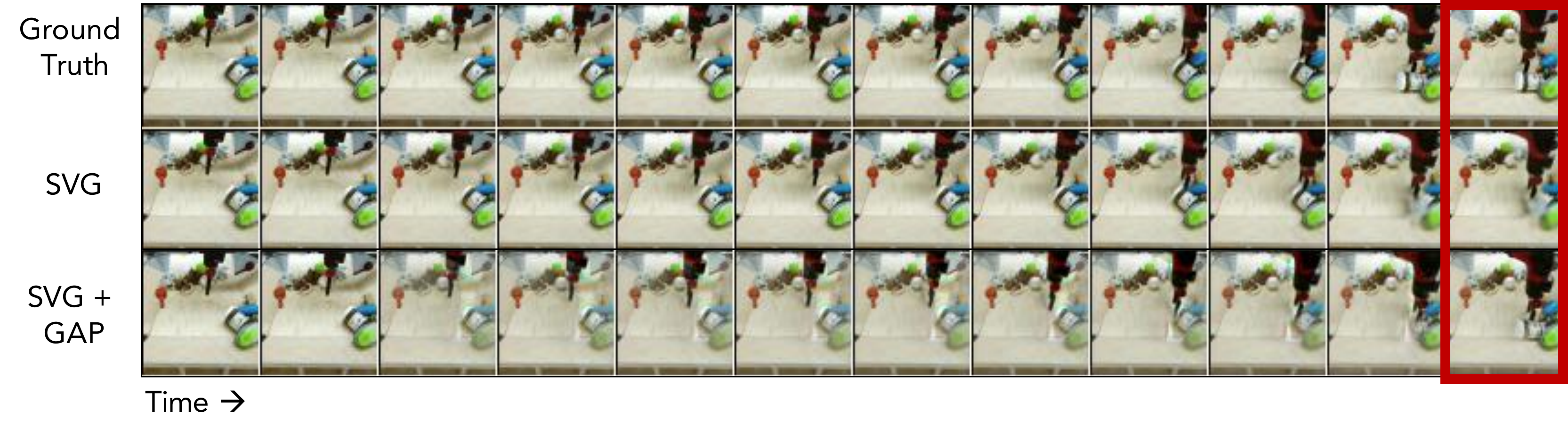}}
\centerline{\includegraphics[width=1.99\columnwidth]{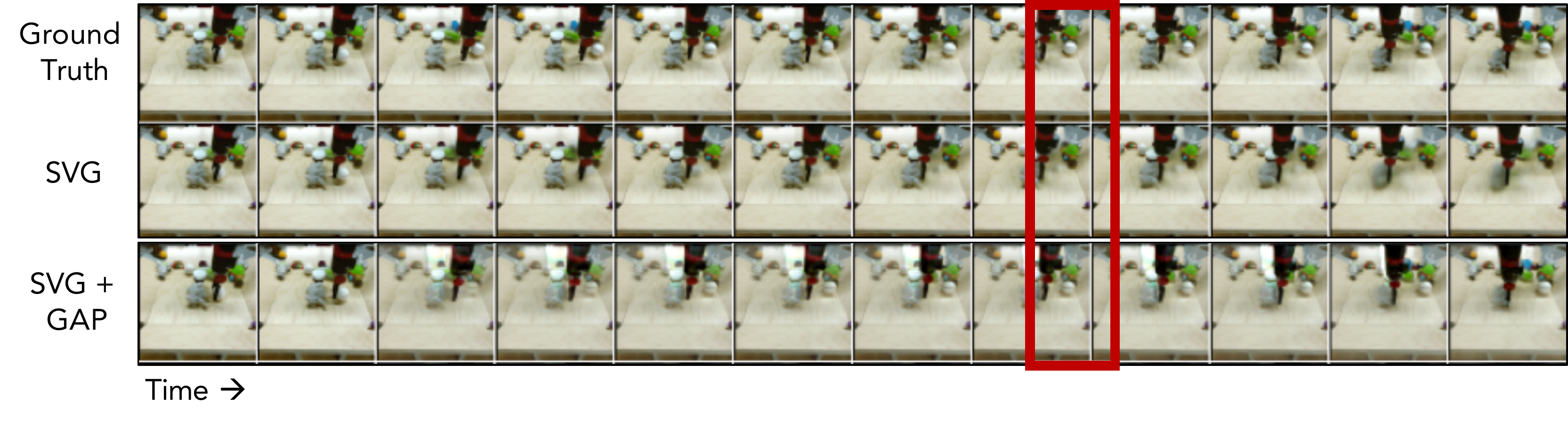}}

\vspace{-0.3cm}
\caption{\small{\textbf{GAP+SVG Video Prediction (BAIR Robot Dataset):} We present additional qualitative examples of action-conditioned SVG with and without GAP on the BAIR robot dataset, predicting on goal-reaching trajectories. Note, in the GAP predictions the goal is added back to the predicted goal-state residual. In this case the goal is the rightmost frame. In the top example, we see that GAP more effectively models the water gun, while normal SVG blurs it out. In the middle example, we see that by using the goal information, GAP is able to more effectively model the jar. In the bottom example we see that SVG blurs out a background object, while GAP does not.}   }
\vspace{-0.3cm}
\label{qual_bair_extra}
\end{figure*}

\end{document}